\documentclass{article}

\PassOptionsToPackage{numbers, compress}{natbib}
\usepackage{appendix}
\usepackage{titletoc}

\usepackage[preprint]{neurips_2025}

\usepackage[utf8]{inputenc} 
\usepackage[T1]{fontenc}    
\usepackage{hyperref}       
\usepackage{url}            
\usepackage{booktabs}       
\usepackage{enumitem}
\usepackage{amsfonts}       
\usepackage{nicefrac}       
\usepackage{microtype}      
\usepackage{xcolor}         

\usepackage{microtype}
\usepackage{graphicx}
\usepackage{subfigure}
\usepackage{booktabs} 
\usepackage{wrapfig}

\usepackage{hyperref}

\usepackage{amsmath}
\usepackage{amssymb}
\usepackage{mathtools}
\usepackage{amsthm}
\usepackage{amsmath,amsfonts,bm}

\definecolor{uclablue}{RGB}{39, 116, 174}
\definecolor{bigaired}{RGB}{156, 0, 0}

\usepackage{url}
\usepackage{multicol}
\usepackage{graphicx}
\usepackage{algorithm}
\usepackage{algorithmic}
\usepackage{wrapfig}
\usepackage{lipsum} 
\usepackage{float}

\usepackage{amssymb}
\usepackage{acronym}
\usepackage[english]{babel}
\usepackage{amsthm}
\usepackage{booktabs}

\usepackage[misc]{ifsym}

\usepackage{natbib}
\usepackage[capitalize,noabbrev]{cleveref}
\newcommand\DoToC{%
  \startcontents
  \hypersetup{linkcolor=black}
  \printcontents{}{1}{\textbf{\Large Contents}\vskip5pt\hrule\vskip5pt}
  \vskip3pt\hrule\vskip5pt
}
\theoremstyle{plain}
\newtheorem{theorem}{Theorem}[section]
\newtheorem{proposition}[theorem]{Proposition}
\newtheorem{lemma}[theorem]{Lemma}

\theoremstyle{definition}

\theoremstyle{remark}

\usepackage[textsize=tiny]{todonotes}

\def\eqref#1{equation~\ref{#1}}

\def\1{\bm{1}}

\def\mQ{{\bm{Q}}}

\DeclareMathAlphabet{\mathsfit}{\encodingdefault}{\sfdefault}{m}{sl}
\SetMathAlphabet{\mathsfit}{bold}{\encodingdefault}{\sfdefault}{bx}{n}

\def\sX{{\mathbb{X}}}

\newcommand{\E}{\mathbb{E}}

\newcommand{\R}{\mathbb{R}}

\DeclareMathOperator*{\argmin}{arg\,min}

\renewcommand{\algorithmiccomment}[1]{\textcolor{blue}{\textit{/* #1 */}}}

\newcommand\ie{{i.e.}}
\newacro{CTDMC}[CTDMC]{continuous time discrete Markov chain}
\newacro{CTDMB}[\textbf{DMB}]{\textbf{Discrete Markov Bridge}}

\title{Discrete Markov Bridge}

\author{%
Hengli Li$^{1,2,3}$ \\ \texttt{lihengli@stu.pku.edu.cn} \\
\And Yuxuan Wang$^{2,3}$ \\ \texttt{wangyuxuan1@bigai.ai} \\ 
\And Song-Chun Zhu$^{1,2,3,4}$ \\ \texttt{s.c.zhu@pku.edu.cn}\\
\And Ying Nian Wu$^{5}$\textsuperscript{\Letter} \\ \texttt{ywu@stat.ucla.edu} \\
\And Zilong Zheng$^{2,3}$\textsuperscript{\Letter}\\ \texttt{zlzheng@bigai.ai}\\ 
    \AND
    \normalfont 
    $^1$ Institute of Artificial Intelligence, Peking University \\
    $^2$ NLCo Lab, Beijing Institute for General Artificial Intelligence \\ 
    $^3$ State Key Laboratory of General Artificial Intelligence \\
    $^4$ Department of Automation, Tsinghua University \\
    $^5$ University of California, Los Angeles \\
}

\begin{document}

\maketitle

\begin{abstract}

Discrete diffusion has recently emerged as a promising paradigm in discrete data modeling. However, existing methods typically rely on a fixed-rate transition matrix during training, which not only limits the expressiveness of latent representations—a fundamental strength of variational methods—but also constrains the overall design space. To address these limitations, we propose \textbf{Discrete Markov Bridge}, a novel framework specifically designed for discrete representation learning. Our approach is built upon two key components: \emph{Matrix}-learning and \emph{Score}-learning. We conduct a rigorous theoretical analysis, establishing formal performance guarantees for \emph{Matrix}-learning and proving the convergence of the overall framework. Furthermore, we analyze the space complexity of our method, addressing practical constraints identified in prior studies. Extensive empirical evaluations validate the effectiveness of the proposed \textbf{Discrete Markov Bridge}, which achieves an Evidence Lower Bound (ELBO) of \textbf{1.38} on the Text8 dataset, outperforming established baselines. Moreover, the proposed model demonstrates competitive performance on the CIFAR-10 dataset, achieving results comparable to those obtained by image-specific generation approaches.\footnote{Implementation code is available at \url{https://github.com/Henry839/Discrete-Markov-Bridge}.} 
%

\end{abstract}
\section{Introduction}

\label{Introduction}
A fundamental question in generative modeling is estimating an underlying distribution, \(\mu\), from observed data and subsequently generating new samples from this distribution. Among the various generative models proposed, diffusion models have exhibited remarkable performance in both continuous \cite{song2021scorebased, ho2020denoising} and discrete domains \cite{campbell2022continuous, lou2024discrete}, demonstrating their versatility and effectiveness in diverse applications. These models effectively capture complex data distributions, enabling high-quality sample generation in various applications. However, despite their strong connection to variational models \cite{kingma2022autoencoding, oord2018neuraldiscreterepresentationlearning}, which are known for their impressive generative capabilities, diffusion models have yet to integrate the latent encoding ability inherent to variational approaches. Specifically, in the discrete domain, the noise rate transition matrices within discrete diffusion models are fixed and constrained, resulting in a limited design space and reduced expressive capacity. To the best of our knowledge, only the Absorb and Uniform Matrix \cite{campbell2022continuous, lou2024discrete, austin2021structure} have been considered in computations due to their simplicity in handling exponential term calculations.

In this study, we challenge the convention of using predefined static matrix in discrete modeling by introducing a novel approach, termed the \textbf{Discrete Markov Bridge (DMB)}, which aims to integrate the strengths of variational methods with discrete diffusion models,  offering a more robust and efficient solution for complex discrete-state systems. This methodology seeks to enhance the modeling capabilities by leveraging the theoretical foundations of variational inference within the framework of discrete diffusion processes. Specifically, DMB is structured as a bidirectional two-stage learning algorithm. It comprises a forward variational process, \ie, \textit{Matrix}-learning, that maps the data distribution to a learned distribution, followed by a backward decoding process, \ie, \textit{Score}-learning, that reconstructs the data distribution from the learned representation.

In the \emph{Matrix}-learning process, we propose a novel parameterized rate transition matrix that enhances the flexibility of the overall algorithm. This refinement allows for greater adaptability and improved performance in dynamic learning environments. The rate transition matrix is designed to be diagonalizable, ensuring high spatial efficiency while facilitating the rapid computation of matrix exponentials. 
On the other hand, in the \emph{Score}-learning process, a neural network is employed to model the concrete score \cite{lou2024discrete, meng2023concretescorematchinggeneralized}. This score serves a crucial role in the derivation of the backward rate transition matrix. As for the sampling procedure, the rate transition matrix derived from the \emph{Matrix}-learning process and the neural network obtained from the \emph{Score}-learning process are jointly employed to solve the backward differential equation.

Within this framework, a broad spectrum of tasks can be effectively addressed. For discrete data modalities such as text, the model supports non-autoregressive generation, following the approach outlined in \cite{gu2022nonautoregressive}. In this work, we demonstrate that our proposed method surpasses the performance of the previously established SEDD model \cite{lou2024discrete}. For image data, the model can be integrated with a VQ-VAE architecture \cite{oord2018neuraldiscreterepresentationlearning}, yielding performance on par with that of DDPM when evaluated on the CIFAR-10 dataset. \looseness=-1

We summarize our contributions as follows:

\begin{itemize}[leftmargin=*, topsep=0pt, noitemsep]

\item \textbf{Novel Framework for Discrete Data (\Cref{sec: ctdmb}):} We introduce the Discrete Markov Bridge, a new variational framework for learning discrete representations. By leveraging a variational formulation, this approach provides a novel method for modeling complex discrete data. 
\item \textbf{Theoretical Guarantee (\Cref{sec: theory}):} We present a theoretical guarantee for the \emph{Matrix}-learning process, covering both its validity and accessibility. Furthermore, we provide a comprehensive analysis of the entire framework, culminating in a formal convergence proof.
\item \textbf{Addressing Practical Issues (\Cref{sec: experiment}):} Building on the theoretical insights established earlier, we propose a computationally efficient matrix to tackle the practical challenges discussed in \Cref{sec: experiment}. We then evaluate the model’s performance through experiments, demonstrating that it outperforms baseline methods in text modeling and provides competitive image modeling results.
\end{itemize}

\begin{figure*}[t!]

    \centering
    \includegraphics[width=.8\linewidth]{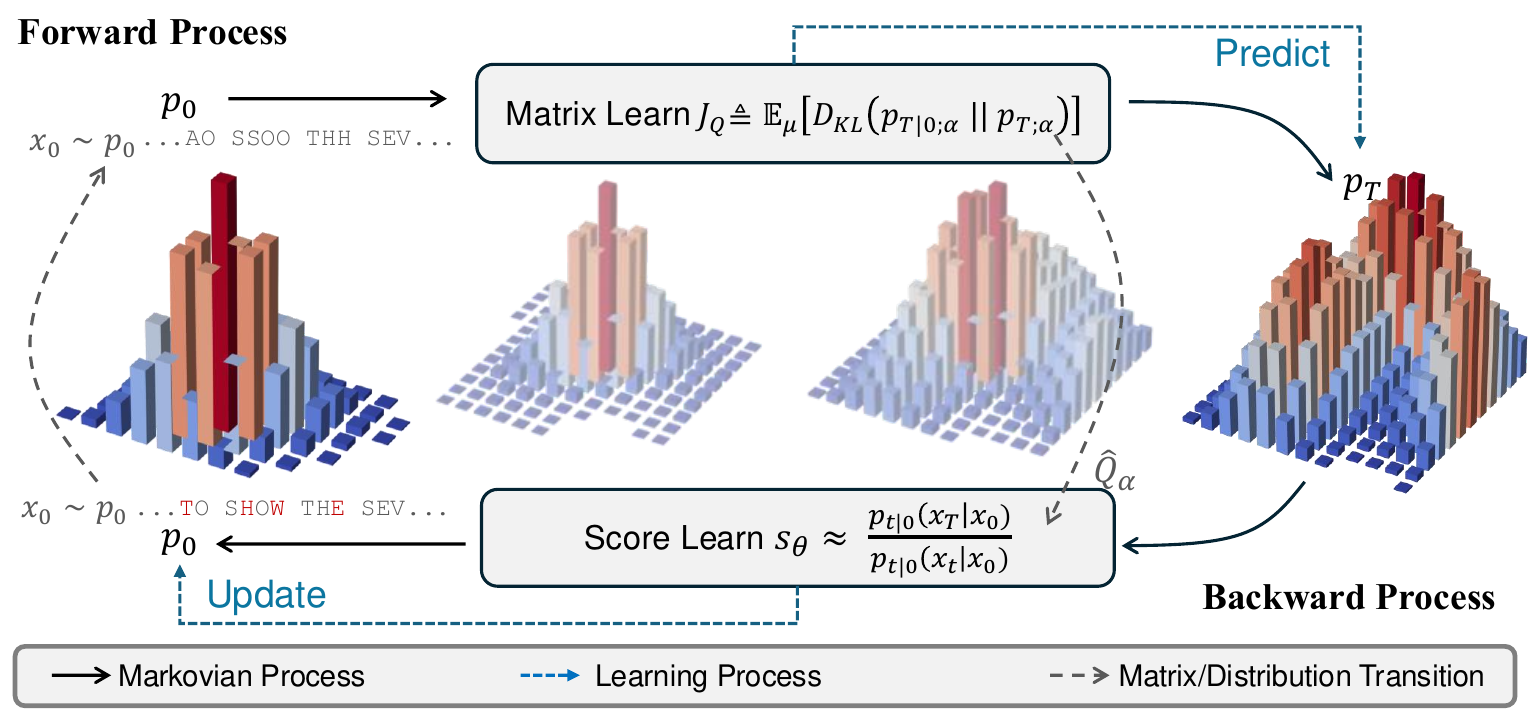}
    \caption{Overview of the \ac{CTDMB} framework. \ac{CTDMB} consists of two component: the \emph{Matrix}-learning and the \emph{Score}-learning. The \emph{Matrix}-learning process is designed to learn an adaptive transition rate matrix, which facilitates the estimation of an adapted latent distribution. Concurrently, the \emph{score}-learning process focuses on estimating the probability ratio necessary for constructing the inverse transition rate matrix, thereby enabling the reconstruction of the original data distribution.}
    \label{fig: DMB}
\end{figure*}

\section{Preliminaries and Related Works}\label{sec:related_work}
\subsection{Continuous-Time Discrete Markov Chain}

Let $\mathbb{X} = \{1,2, \ldots, n\}$ denote a finite state space, where $n \in \mathbb{R}$. A \ac{CTDMC} defined on $\mathbb{X}$ is represented as $\{X(t) \mid t \in \mathbb{R}, X(t) \in \sX\}$. For convenience, we use the notation $X_t \triangleq X(t)$. The probability of transitioning from state $x \in \mathbb{X}$ at time $t$ to state $y \in \mathbb{X}$ at time $t + s$ is denoted as  $p_{t+s|t}(y|x) \triangleq P(X_{t + s} = y \mid X_t = x)$. Similarly, the probability that \(X_t\) takes state \(x\) at time \(t\) is expressed as  $p_t(x) \triangleq P(X_t = x)$. The probability distribution over the state space at time \(t\) is then given by the vector  
$p_t \triangleq (p_t(1), p_t(2), \ldots, p_t(n))$. The core component to describe a continuous time discrete Markov chain is the rate transition matrix. We defined the rate transition probability as follows:
\begin{equation}
q_{t}(x,y) \triangleq \frac{dp_{t+s|t}(y|x)}{ds}  = \lim_{\Delta s \to 0} \frac{p_{t+s|t}(y|x) - p_{t|t}(y|x)}{\Delta s} = \lim_{\Delta s \to 0} \frac{p_{t+s|t}(y|x) - \delta_{x}(y) }{\Delta s},
\notag
\end{equation} where $\delta_x(y)$ is the Dirac delta function. The Forward Kolmogorov Equation can be written as 
$
    \frac{d p_t}{dt} = p_t Q^{(t)}
$.
The notation $Q^{(t)}_{x,y} \triangleq q_{t}(x,y)$, for all $x, y \in \mathbb{X}$, denotes the rate transition matrix at time $t$. The subscripts $x$ and $y$ indicate the row and column indices, respectively. Each rate transition matrix satisfies the conditions: the sum of each row must be zero, and all off-diagonal entries must be non-negative. Formally, this is expressed as $\sum_{y} Q_{x,y} = 0$ for all $x$ and $Q_{x,y} \geq 0$ for all $y \neq x$.

\subsection{Related Works}
\paragraph{Prior Learning}
Leveraging a prior is a longstanding paradigm in machine learning. In the field of natural language processing, for example, training typically begins with pretrained language models \cite{liu2019roberta, devlin2019bert, touvron2023llama, radford2019language, lan2020albert, hu2021lora, vaswani2023attention}. Likewise, pretrained models are highly valued in computer vision \cite{he2015deep}. In our approach, the concept of a prior is equally fundamental: the forward process adaptively refines this prior based on the evolving training dynamics of the backward process.

\paragraph{Discrete Diffusion Models} Diffusion models \cite{ho2020denoising, song2021scorebased, song2022denoising, sohldickstein2015deep} add noise to data and use a denoiser for reconstruction, achieving success in image tasks and gaining traction in discrete domains like natural language \cite{li2022diffusionlm, lou2024discrete, campbell2022continuous, gulrajani2023likelihoodbased, sun2023scorebased, dieleman2022continuous,nie2025largelanguagediffusionmodels}. Some methods map discrete data to continuous space \cite{li2022diffusionlm, gulrajani2023likelihoodbased}, introducing rounding errors, while others operate directly in discrete space but impose rigid, non-learnable noise structures \cite{campbell2022continuous, lou2024discrete}. In the continuous domain, trainable Gaussian parameters improve flexibility \cite{kingma2023variationaldiffusionmodels}, but no such method exists for discrete diffusion, where Gaussian distributions also remain restrictive. Moreover, masked discrete diffusion models struggle to learn temporal dependencies \cite{zheng2024maskeddiffusionmodelssecretly}. \looseness=-1

\paragraph{Flow Models}
Flow-based models \cite{rezende2016variational, kingma2018glow, liu2022flow, satorras2022en, albergo2023stochastic, trockman2021orthogonalizing} constitute a prominent class of machine learning models characterized by their ability to perform reversible transformations on data representations. In contrast to conventional flow models, which rely on transformation paths predefined by human designers \cite{albergo2023stochastic, liu2022flow}, our approach autonomously learns these paths, enhancing adaptability and expressiveness in data modeling. \looseness=-1

\section{Discrete Markov Bridge}

\label{sec: ctdmb}


The target distribution, denoted as \(\mu \in \mathbb{R}^n\), is a probability vector, meaning that its elements are non-negative and collectively sum to one. As shown in \Cref{fig: DMB}, our objective is to estimate the distribution at one endpoint of the Markov chain, denoted as $p_0$, such that $p_0 \approx \mu$. The other endpoint, denoted as $p_T$, serves as the distribution for the latent variables or prior. To achieve the specified objectives, the proposed \ac{CTDMB} framework is structured into two distinct components: \emph{Matrix Learning} and \emph{Score Learning}. 

The \emph{Matrix}-learning serves as a forward bridge, facilitating the transition from $\mu$ to the latent distribution. Conversely, the \emph{Score}-learning function delineates a reverse pathway from the latent distribution back to $\mu$, leveraging the groundwork established by the \emph{Matrix}-learning process. This dual-function framework ensures a comprehensive bidirectional understanding of the data structure, enhancing the robustness of the analytical model.

The structure of the \ac{CTDMB} is demonstrated in \Cref{algo: training}. This pseudocode illustrates two nested while loops that operate within the overarching while loop governing the training epochs. Each of these nested loops corresponds to a distinct learning stage within the framework, effectively organizing the training process into two phases. We list the following theorem to ensure the reversibility of the forward and backward Markovian processes.

\begin{theorem}[Reversibility \citep{campbell2022continuous, lou2024discrete}]
\label{thm: reversibility}
    Given the Forward Kolmogorov Equation of a \ac{CTDMC}:
    \begin{equation}
    \label{eq: forward}
        \frac{dp_t}{dt} = {p_t}Q^{(t)}
    \end{equation}
    There exists a reverse \ac{CTDMC} with Forward Kolmogorov Equation:
    \begin{equation}
    \label{eq: reverse}
        \frac{dp_{T-t}}{dt} = p_{T-t} \hat{Q}^{(T-t)} \text{ ,where $\hat{Q}^{(t)}_{x,y} = \frac{p_t(y)}{p_{t}(x)} Q^{(t)}_{y,x}$}
    \end{equation}
    
\end{theorem}

This theorem elucidates the reverse form of a \ac{CTDMC}, proposing that knowledge of the probability ratio enables the derivation of a reversal of the original Markov chain that is almost everywhere equivalent. This assertion underscores the theoretical framework necessary to comprehend the conditions under which the reverse process mirrors the dynamics of the forward stochastic process.

We structure the learning process of the framework by employing the continuous-time Evidence Lower Bound (ELBO) as an alternative optimization objective to Maximum Likelihood Estimation (MLE). In the \ac{CTDMB} framework, both \emph{Matrix}-learning and \emph{Score}-learning collaboratively optimize distinct segments of the full bound through their respective subprocesses. 

\subsection{\emph{Matrix}-Learning}
\label{subsec: matrix learning in DMB}

In the \emph{Matrix}-learning process, our primary objective is to estimate the rate transition matrix $Q_{\alpha}$, where $\alpha$ denotes the set of model parameters. For simplicity, we assume that the forward rate transition matrix at time $t$, denoted $Q^{(t)}_{\alpha}$, is given by $\sigma(t)Q_{\alpha}$. Furthermore,  we employ the following $Q_\alpha$:\looseness=-1
\begin{equation}
Q_{\alpha} = 
A\begin{bmatrix}
    -\sum \limits_{i = 1}^{n - 1}a_i & a_1 & \ldots & a_{n-2} & a_{n-1}\\
    0 & -\sum \limits_{i = 2}^{n - 1}a_i & \ldots & a_{n-2} & a_{n-1}\\
    \ldots & \ldots & \ldots & \ldots & \ldots \\
    0 & 0 & \ldots & - a_{n-1} & a_{n-1} \\
    0 & 0 & \ldots & 0 & 0 
\end{bmatrix}A^{-1} := A H A^{-1}
\label{eq: form of Q alpha}
\end{equation}
, where $\{a_1, a_2, \ldots, a_{n-1}\} = \alpha$ are parameters for learning, $A, A^{-1}$ are fixed predefined permutation matrices and $H$ is introduced to stand for the upper-triangle matrix. 
The derivation and underlying rationale for utilizing this matrix are detailed in \Cref{sec: theory} and further explored in \Cref{subsec: practical}.
Another essential component of this process is $\mu$, which is approximated using the currently predicted $p_0$ obtained through \emph{Score}-learning as a prior (see \Cref{subsec: score learning}). By integrating \Cref{eq: forward} from time $0$ to time $t$, the following equation can be derived:
\begin{equation}
\label{eq: evolve}
p_t = p_{0}\exp\{ \int_{0}^{t} \sigma(s) ds Q_{\alpha}\}
\end{equation}

Note that the exponential in the formula is a matrix exponential. The training procedure aims to minimize a portion of the variational bound, leading to the following objective function $J_Q$:
\begin{equation}
    J_Q \triangleq \mathbb{E}_\mu[D_{KL}(p_{T|0;\alpha}||p_{T;\alpha})],
    \label{eq:loss_Q}
\end{equation}
 where the conditional probability distribution $p_{T|0;\alpha}$ is given by the rows of $\exp\{\int_{0}^{T} \sigma(s) ds Q_{\alpha}\}$: 
\begin{equation}
    p_{T|0; \alpha}(x_T|x_0) = \exp\{\int_0^T \sigma(s) ds Q_\alpha\}_{x_0, x_T}
\end{equation}

The final distribution $p_{T;\alpha}$ is obtained by multiplying the initial distribution $p_{0}$ with the conditional distribution, as presented in \Cref{eq: evolve}, evaluated at time $t = T$.

\begin{algorithm}[H]
\caption{Training Algorithm of \ac{CTDMB}}
\textbf{Input: } Target discrete data $X\sim \mu$
\begin{algorithmic}[1]
 \STATE Initialize $p_0, p_T \gets$ \emph{random\_init}()
 \WHILE{\textit{not converge}} 
  \STATE Sample a batch of discrete instance $X_0 \sim \mu$.
  \COMMENT{Data for the two learning processes.} \\
\COMMENT{Matrix Learning}
 \STATE $step \gets 0$
 \WHILE{$step \leq max\_step \And \mathcal{L}_Q \geq \epsilon_Q$}
 \STATE Update $Q_\alpha, \mathcal{J}_Q$ according to Eqn. (\ref{eq:loss_Q}) and predict $p_T$ using Eqn. (\ref{eq: evolve}) at $t=T$.

  \STATE $step \gets step + 1$
 \ENDWHILE \\
 \COMMENT{Score Learning}
  \STATE $step \gets 0$
  \WHILE{$step \leq max\_step \And \mathcal{J}_{score} \geq \epsilon_{score}$}
 \STATE Update $s_\theta, \mathcal{J}_{score}$ \textit{w.r.t.}  current $Q_\alpha$ using Eqn. (\ref{eq: score entropy loss}).
 
 \STATE $step \leftarrow step + 1$
 \ENDWHILE
\STATE Predict updated $p_0$ that estimates $\mu$ using Eqn. (\ref{eq: mu estimation}). \COMMENT{Used for Matrix Learning}
 \IF {$ \mathcal{J_Q} + \mathcal{J}_{score} < \epsilon$}
 \STATE \textit{converge} $\gets$ TRUE
 \ENDIF
 \ENDWHILE

\end{algorithmic}
\caption{Training Algorithm of the \ac{CTDMB}}
\label{algo: training}
\end{algorithm}

\subsection{\emph{Score}-learning}
\label{subsec: score learning}
\emph{Score}-learning constitutes a reverse process of \emph{Matrix}-learning. It is noted that in \Cref{thm: reversibility}, the reverse rate transition matrix adheres to the following relationship:
\begin{equation}
\label{eq: reverse rate transition matrix}
    \hat{Q}^{(t)}_{x, y} = \frac{p_t(y)}{p_t(x)} Q_{x, y} \sigma(t)
\end{equation}
Consequently, while \emph{Matrix}-learning handles the forward rate transition matrix $Q$, \emph{Score}-learning focuses on managing the remaining part, i.e $\frac{p_t(y)}{p_t(x)}$. A learnable model $s_\theta(x_t,t)_y$
is designed to model the ratio, and the main part of the continuous time Evidence Lower Bound (ELBO) \citep{campbell2022continuous, lou2024discrete, kingma2022autoencoding} is leveraged as the training objective, denoted as $J_{score}$:
\begin{equation}
\label{eq: score entropy loss}
\resizebox{\linewidth}{!}{

$\int_{0}^{T} \mathbb{E}_{x_0 \sim \mu, x_t \sim p_{t|0}}\bigl[\sum \limits_{y \neq x_t} Q^{(t)}_{y, x_t}\Biggl(s_\theta(x_t, t)_{y} - \frac{p_{t|0}(y|x_0)}{p_{t|0}(x_t|x_0)} \\ 
 + \frac{p_{t|0}(y|x_0)}{p_{t|0}(x_t|x_0)} \bigl( \log(\frac{p_{t|0}(y|x_0)}{p_{t|0}(x_t|x_0)}) - \log s_\theta(x_t, t)_y \bigl)\Biggl)\bigl] dt$
 }
\end{equation}

To provide a comprehensive understanding, we present the complete ELBO as follows, demonstrating how \emph{Matrix}-learning and \emph{Score}-learning collaboratively contribute to minimizing the ELBO bound.
\begin{equation}
\label{ineq: full ELBO bound}
\mathbb{E}_{x_0 \sim \mu}[-\log p_{0; \theta}(x_0)] \leq J_{score} + J_Q.
\end{equation}
\paragraph{Estimating $\mu$}
The estimation of $\mu$ is expressed as \Cref{eq: mu estimation}. The equation below is derived under the Euler method and can be generalized to other ODE-solving methods. Suppose the inference time process is partitioned as: $[0, t_1], [t_1, t_2], \ldots, [t_n, T]$. By Bayesian rules:
\begin{equation}
\label{eq: mu estimation}
        \mu(x_0) \approx p_0(x_0) = \mathbb{E}_{X_T, X_{n}, \ldots X_{1}}[p_{0|1}(x_0 | x_1)].
\end{equation}
Under the guidance of \Cref{eq: mu estimation}, the sampling process begins with drawing $x_T$, followed by obtaining $x_n$ through the conditional distribution $p_{t_n|T}(x_n | X_T = x_T)$. This procedure continues iteratively, generating $x_{n-1}$, and proceeding sequentially until the complete sequence $\{X_T, X_n, \dots, X_1\}$ is sampled. Subsequently, the conditional probability $p_{0|1}(x_0 | x_1)$ is determined. By repeating this process multiple times and averaging the sampled probabilities, an estimation can be obtained by approximating the expectation with the empirical mean.

\subsection{Sampling}
The sampling process is done under the cooperation of \emph{Matrix}-learning and \emph{Score}-learning in a similar way as estimating $\mu$. The reverse rate transition matrix is calculated as \Cref{eq: reverse rate transition matrix}, and an ode-solving method such as the Euler method can be further applied to solve \Cref{eq: reverse}. Noticed that, as shown in line 15 of \Cref{algo: training}, the sampling process is performed every time after the \emph{Score}-learning process to gain the estimation of $\mu$ and samples for evaluation.

\section{Theoretical Foundations}
\label{sec: theory}
\subsection{Validity and Accessibility of \emph{Matrix}-learning}
The validity and accessibility of the backward process are established by \Cref{thm: reversibility}. In this subsection, we extend our analysis to the same aspects of the \emph{Matrix}-learning process. Specifically, \textbf{validity} concerns the ultimate state of the forward process and whether it remains confined within a well-defined domain, i.e., whether a probability distribution transforms into another valid probability distribution. \textbf{Accessibility}, on the other hand, pertains to the ability of the process to transition between any two arbitrary discrete distributions, thereby characterizing the reachability and adaptability of the \emph{Matrix}-learning process.

\paragraph{Validity} \Cref{prop: conservation}, presented below, establishes that any transformation originating from a probability distribution must result in another probability distribution. This theorem guarantees that, despite the presence of errors in the learning process, the outcome remains a valid probability distribution. For a detailed proof, refer to \Cref{supp: sec: proof_conservation}.
\begin{proposition}[Conservation of the Sum]
\label{prop: conservation}
For two arbitrary vectors $\phi, \mu \in \R^{n}$, rate transition matrix $Q \in \mathcal{R}^{n \times n}$, if $\phi = \mu \exp\{Q\}$, then 
$$
\sum_{i = 1}^n \phi[i] = \sum_{i = 1}^n \mu[i]
$$
\end{proposition}

\paragraph{Accessibility}
\Cref{thm: accessibility} ensures that any two probability distributions are accessible in the forward process. Consequently, this implies that the optimality of \textit{Matrix}-learning can be achieved, provided the presence of a strong optimizer.

\begin{theorem}[Accessibility]
\label{thm: accessibility}
For two arbitrary discrete distributions $p,q \in \mathbb{R}^{n}$, there exists a rate transition matrix $Q \in \mathbb{R}^{n \times n}$ such that:
\begin{equation}
\label{eq: accessibility equation}
   p = q e^Q
\end{equation}

\end{theorem}

The central idea of the proof is to construct a specialized matrix that possesses strong representational capacity while remaining computationally manageable within the framework of matrix exponentiation. The designed matrix, which is depicted in Lemma \ref{lemma: Q}, is an upper triangle matrix with the vanished sum of rows. A remarkable characteristic of this matrix is its elegant eigendecomposition form, which presents a well-structured and analytically convenient representation. Its eigenmatrix is an all-one upper triangular matrix, as shown in \Cref{lemma: Q}.

\begin{lemma}
\label{lemma: Q}
Let matrix $Q \in \mathbb{R}^{n \times n}$ and hold the following form:
\[
Q = H
\]

, where $H$ is defined in \Cref{eq: form of Q alpha}, then $Q$ can be diagonalized in the following form:
$$
Q = U \Lambda U^{-1}
$$,where $U = \begin{bmatrix}
   1  & 1  & \ldots & 1\\
    0  & 1  & \ldots & 1\\
    \ldots  & \ldots & \ldots & \ldots \\
    0 & 0  & \ldots & 1
\end{bmatrix}$, $\Lambda = diag(\{-\sum \limits_{i = 1}^{n - 1}a_i, -\sum \limits_{i = 2}^{n - 1}a_i, \ldots, - a_{n-1}, 0\})$.
\end{lemma}
There are two key observations regarding the $Q$ matrix. First, it contains only $n-1$ parameters, which constitute the minimal set necessary to solve \Cref{eq: accessibility equation}. This sufficiency implies that the solution derived for the $Q$ matrix is unique. Second, the matrix retains nonzero elements exclusively in its upper triangular portion, implying that each element can transition only to those with a larger index. This observation raises an additional consideration: for effective state transitions, the matrix must allocate sufficient ``mass" or probability. Consequently, a matrix is required to appropriately adjust the indices of elements within the finite set \(\mathbb{X}\), as shown in \Cref{lemma:ineq_of_mass}. \Cref{lemma:ineq_of_mass} establishes that, after a permutation, the cumulative probability at each element of the initial distribution in the transition process is greater than or equal to that of the target distribution. This guarantees that elements with surplus probability can redistribute their excess, while those with a deficiency can receive the necessary adjustments, ensuring a balanced transformation.
\begin{lemma}
\label{lemma:ineq_of_mass}
For arbitrary distribution $p, q \in \mathbb{R}^{n}$, there exists an permutation matrix $A$ such that:
\begin{equation}
   \frac{p_1'}{q_1'} \leq \frac{p_1' + p_2'}{q_1' + q_2'} \leq \ldots \leq \frac{\sum \limits_{i = 1}^k p_i'}{\sum \limits_{i = 1}^k q_i'} \leq \ldots \leq \frac{\sum \limits_{i = 1}^n p_i'}{\sum \limits_{i = 1}^n q_i'} = 1
\label{eq: ineq_of_mass} 
\end{equation}
where $p' = p A$, $q'=qA$, $p_i'$ is the $i$-th entry of $p'$, and $q_i'$ is the $i$-th entry of $q'$.
\end{lemma}

\begin{lemma}
\label{lemma: rate}
    Let $Q \in \mathbb{R}^{n \times n}$ be a rate transition matrix, $A \in \mathbb{R}^{n \times n}$ be a permutation matrix, then $AQA^{-1}$ is a rate transition matrix.
\end{lemma}

By integrating the lemmas above, we aim to establish the proof of \Cref{thm: accessibility}. A comprehensive derivation of these lemmas and the theorem is provided in \Cref{supp: sec: proof_accessibility}.

\subsection{Convergence}
As discussed earlier, the \ac{CTDMB} framework operates as a two-step learning algorithm, necessitating a thorough examination of its convergence properties. In this section, we present a formal theorem that establishes the convergence guarantee for the entire algorithm.
The convergence problem is nontrivial, as the \textit{Score}-learning process does not merely constitute a direct inversion of the \emph{Matrix}-learning process. The discrepancy arises because the score model $s_\theta$ is trained under the supervision of the distribution $\mu$, rather than $p_0^{(k)}$, where $k$ denotes the epoch number. To be specific, we have
\begin{proposition}[Supervision of \textit{Score}-learning]
    Suppose $Q_t$'s elements are non-zeros, the training objective is depicted as in \Cref{eq: score entropy loss}, then the optimality of the score model $s_{\theta^*}(x_t,t)_y$ satisfies:
    \begin{equation}
        s_{\theta^{*}}(x_t,t)_y = \mathbb{E}_{x_0 \sim \mu_{0|t}(\cdot|x_t)}[\frac{p_{t|0}(y|x_0)}{p_{t|0}(x_t|x_0)}] \\ \nonumber
        = \frac{\sum \limits_{x_0}\mu(x_0) p_{t|0}(y|x_0)}{\sum \limits_{x_0}\mu(x_0) p_{t|0}(x_t|x_0)}
    \end{equation}
\end{proposition}

The proposition presented above illustrates the influence of $\mu$ on the training process and underscores the challenge of convergence arising from the absence of $p_0^{(k)}$. A detailed proof of this proposition can be found in \Cref{supp: sec: proof_supervision}.

Under the assumption that each process achieves optimality, the following theorem establishes the convergence of \ac{CTDMB} from the perspective of KL divergence, thereby demonstrating the validity of the overall \ac{CTDMB} framework. Moreover, given our primary focus on the algorithmic aspects, this assumption is justified, consistent with prior work that introduces new frameworks, such as \citet{goodfellow2014generativeadversarialnetworks}. Notably, although the training objective of the \emph{Score}-learning process is the continuous ELBO bound, the theorem presented below can be generalized to encompass a broader class of objectives. This generalization suggests the potential for designing improved training objectives within our framework.

\begin{theorem}[Convergence of the algorithm]
\label{thm: convergence}
    If we assume optimality is achieved in every epoch of the Matrix-learning process and the Score-learning process, and we denote the $k$-th epoch estimation of $\mu$ as $p_0^{(k)}$, then 
    $\lim_{k \to \infty} D_{KL}(\mu||p_0^{(k)})$ converges.

\end{theorem}



 Please refer to \Cref{sec: proof_convergence} for the proof.


\section{Practical Issues and Experiments}
\label{sec: experiment}
\subsection{High Dimensional Data}
\label{subsec: practical}
In this section, we discuss the practical issues of \ac{CTDMB} by assuming our data coming from a high dimensional space, \ie $\mu \in \mathbb{R}^{d \times n}$, where $n$ is the size of the finite set and $d$ is the number of dimensions. For instance, for textual data, $n$ is the size of the vocabulary and $d$ is the sequence length.

\paragraph{Assumptions.}
\label{subsec: high dimensional data}
When dealing with high-dimensional data, such as textual sequences, the combinatorial explosion in the number of possible sequences imposes prohibitive constraints on both storage and computational efficiency. To address this challenge, certain assumptions are introduced \cite{lou2024discrete, campbell2022continuous, hoogeboom2021argmaxflowsmultinomialdiffusion}:

\begin{itemize}[leftmargin=*, topsep=0pt, noitemsep]
\item Independent Evolution: $p_{T|0; \alpha}(x_T|x_0) = \prod_{i=1}^{d}p_{T|0; \alpha}(x_T^{(i)} | x_0^{(i)})$

\item Independent Terminal: $p_T(x_T) = \prod_{i = 1}^d p(x_T^{(i)})
$
\end{itemize}

The first assumption posits that, during the forward process, each dimension evolves independently. The second assumption asserts that the latent space consists of independent dimensions.

\subsection{Addressing Practical Issues}
Both the \ac{CTDMB} model and discrete diffusion models \citep{lou2024discrete, campbell2022continuous, sun2021image} face significant challenges related to the \( Q \) matrix. In particular, during the \emph{Score}-learning process, the computational efficiency of matrix exponential operations becomes a critical constraint. Furthermore, the \emph{Matrix}-learning process often requires storing the entire \( Q \) matrix, posing substantial concerns regarding space efficiency. These limitations have been the primary reasons restricting previous studies to utilizing only the Uniform and Absorb matrices. 

As Jean le Rond d'Alembert once remarked, \textit{Algebra is generous; she often gives more than is asked of her.} In the context of proving \Cref{thm: accessibility}, we identify a distinct class of matrices, as mentioned in \Cref{subsec: matrix learning in DMB} and further rigorously discussed in \Cref{lemma: Q}. This structured approach not only underscores the theoretical underpinnings but also highlights the practical implications of matrix manipulation in these models. 

\textbf{Efficient Computation of the permutation matrix.} Before proceeding with the analysis of the $Q_\alpha$ matrix, we first outline the computation of the predefined permutation matrix $A$. As illustrated in the assumptions, the evolution of each dimension occurs independently. Consequently, for each dimension, the permutation matrix is computed separately. In accordance with \Cref{lemma:ineq_of_mass}, we assume the denominator to be constant. Therefore, the permutation matrix for the $i$-th dimension satisfies the following inequality:
$$
\mu(X_0^{(i)} = j) \leq \mu(X_0^{(i)} \leq j+1), \forall j \in {1, 2, \ldots, n-1}
$$
The marginal distribution $\mu(X_0^{(i)})$ can be efficiently estimated in the form of a histogram by extracting a subbatch from the dataset. Subsequently, the permutation matrix is computed using a fast sorting algorithm with a time complexity of $O(n \log n)$.

\paragraph{Efficient Computation of Matrix Exponential.} Matrix exponential is difficult to calculate as it's defined through Tylor expansion, however, a property exists:
\begin{proposition}
\label{prop: fast calculation of matrix exponential}
    For a matrix $Q  \in \mathbb{R}^{n \times n}$ and a non-degenerate matrix $D \in \mathbb{R}^{n \times n}$, we have:
    $$
    \exp\{DQD^{-1}\} = D \exp\{Q\}D^{-1}
    $$
\end{proposition} Please refer to \Cref{supp: proof of matrix exponential calculation} for the derivation of \Cref{prop: fast calculation of matrix exponential}.
By \Cref{prop: fast calculation of matrix exponential}, 
\begin{equation}
\begin{aligned}
    \exp\{Q_\alpha\}  = \exp\{(AU)\Lambda_\alpha(AU)^{-1}\} 
    = (AU) \exp\{\Lambda_\alpha\}(AU)^{-1}
\end{aligned}
\end{equation}, where $U$ is the all-one upper triangle matrix, $\Lambda_\alpha$ is a diagonal matrix parameterized by $\alpha$. Therefore, the computation of the matrix exponential is reduced to evaluating the exponential of a diagonal matrix, which is significantly more efficient.

\paragraph{Space Efficiency.} 
For the permutation matrices $A, A^{-1} \in \mathbb{R}^{d \times n \times n}$, a total of $d \times 2n$ parameters are required. Apart from $A, A^{-1}$, the upper triangle matrix can be decomposed into a non-parameterized all-one upper triangle matrix, a parameterized diagonal matrix, and a constant matrix. Consequently, the storage requirement is of the order $O(nd)$ parameters.

\subsection{ELBO Bound Calculation}
\label{subsec: elbo bound evaluation}
As shown in \Cref{ineq: full ELBO bound}, the computation of the full bound necessitates the evaluation of both the $J_{score}$ and the expected Kullback–Leibler (KL) divergence between the evolved distribution and the target distribution, expressed as $\mathbb{E}_{\mu}D_{KL}(P_{T|0}||P_T)$. Under the assumptions outlined within \Cref{subsec: high dimensional data}, we can derive a closed-form expression for computing the KL term: 
\begin{proposition} 
The KL term can be calculated as:
\begin{equation}
D_{KL}\bigl(p_{T|0;\alpha}(x_T|x_0) || p_T(x_T) \bigl) = \sum \limits_{i = 1}^{d}D_{KL}\bigl(p_{T|0;\alpha}(x_T^{(i)} | x_0^{(i)})|| p_T(x_T^{(i)}) \bigl)
\end{equation}
\end{proposition}
\begin{minipage}{0.48\linewidth}
\begin{table}[H]
    \centering
    \caption{The results were tested 1000 times on the Text8 dataset. We adopt the baseline results reported in \citep{lou2024discrete} for comparison. AR: Autoregressive. NAR: Non-autoregressive.}
    \resizebox{\linewidth}{!}{
    \begin{tabular}{l|lr}
        \toprule
       Type  &  Model & BPC ($\downarrow$)\\
       \midrule
       AR & IAF/SCF \citep{ziegler2019latentnormalizingflowsdiscrete} & $1.88$ \\
                     & AR Argmax Flow \citep{hoogeboom2021argmaxflowsmultinomialdiffusion} & 1.39 \\
                     & Discrete Flow \citep{Tran2019discreteflows} & 1.23 \\
       \midrule
       NAR & SEDD Uniform \citep{lou2024discrete} & $\leq 1.47$ \\
                          & SEDD Absorb \citep{lou2024discrete}& $\leq 1.39$ \\
                          & D3PM Uniform \citep{austin2021structure} & $\leq 1.61$ \\
                          & D3PM Absorb \citep{austin2021structure} & $\leq 1.45$ \\
                          & Mult. Diffusion \citep{hoogeboom2021argmaxflowsmultinomialdiffusion} & $\leq 1.72$ \\
                          & MAC \citep{shih2022traininginferenceanyorderautoregressive}& $\leq 1.40$ \\
                          & BFN \citep{graves2024bayesianflownetworks}& $\leq 1.41$ \\
                          & \ac{CTDMB} (\textbf{Ours}) & $\leq \textbf{1.38}$ \\
    \bottomrule
    \end{tabular}
    }

    \label{tab:text8}
\end{table}
\end{minipage}
\begin{minipage}{0.48\linewidth}
\begin{table}[H]
\centering
  \caption{CIFAR-10 Results. We report inception score (IS), and Fréchet Inception Distance (FID) score. Results are adopted from \citet{Ho2020ddpm}.}

\resizebox{\linewidth}{!}{
 \begin{tabular}{lcc}
    \toprule
    Model & IS ($\uparrow$) & FID ($\downarrow$) \\
    \midrule
    \textbf{Conditional} \\
    \midrule
    EBM~\citep{du2019implicit} & $8.30$ & $37.9$ \\
    JEM~\citep{grathwohl2020your} & $8.76$ & $38.4$ \\
    BigGAN~\citep{brock2018large} & $9.22$ & $14.73$ \\
    StyleGAN2 + ADA (v1)~\cite{karras2020training} & $\mathbf{10.06}$ & $\mathbf{2.67}$ \\
    \midrule
    \textbf{Unconditional} \\
    \midrule
    Gated PixelCNN~\citep{oord2016conditional} & $4.60$ & $65.93$  \\
    PixelIQN~\citep{ostrovski2018autoregressive} & $5.29$ & $49.46$ \\
    EBM~\citep{du2019implicit} & $6.78$ & $38.2$ \\
    NCSN~\citep{song2019generative} & $8.87\!\pm\!0.12$ & $25.32$ \\
    SNGAN~\citep{miyato2018spectral} & $8.22\!\pm\!0.05$ & $21.7$ \\
    SNGAN-DDLS~\citep{che2020your} & $9.09\!\pm\!0.10$ & $15.42$ \\
    StyleGAN2 + ADA (v1)~\citep{karras2020training} & $\mathbf{9.74} \pm 0.05$ & $3.26$ \\
    DDPM (fixed isotropic)~\citep{Ho2020ddpm} & $7.67\!\pm\!0.13$ & $13.51$  \\
    DDPM ($\mathrm{simple}$)~\citep{Ho2020ddpm} & $9.46\!\pm\!0.11$ & $\mathbf{3.17}$ \\ 
    \midrule
    \textbf{Ours} &  8.64 & 11.63\\
    \bottomrule
  \end{tabular}
  }
\label{tab: cifar10}
\end{table}
\end{minipage}

\subsection{Experiment}
In this section, the performances of \ac{CTDMB} on Text8 and CIFAR-10 are reported.

\paragraph{Best Performance on Text8} We conduct our experiments using the Text8 dataset to evaluate the proposed framework. The experimental results are summarized in \Cref{tab:text8}. To ensure statistical reliability, the model was evaluated across 1,000 independent trials. The primary performance metric, the Evidence Lower Bound (ELBO), was computed following the methodology outlined in \Cref{subsec: elbo bound evaluation}. Our proposed approach, \ac{CTDMB}, achieves a Bits Per Character (BPC) bound of 1.38, surpassing baseline models such as SEDD \cite{lou2024discrete}, a representative discrete diffusion model. Notably, our approach does not modify the vocabulary; in particular, no mask token is introduced. Consequently, when compared to similar methods that also do not incorporate a mask token—such as SEDD Uniform and D3PM Uniform—our approach demonstrates an improvement of approximately 0.1 points. \looseness=-1

\paragraph{Competitive Performance on CIFAR-10}  
Although our approach is not specifically tailored for image modeling tasks, we evaluate its performance on the CIFAR-10 dataset using a VQ-VAE framework~\cite{oord2018neuraldiscreterepresentationlearning}. The quantitative results are presented in \Cref{tab: cifar10}. Our method, \ac{CTDMB}, achieves an Inception Score (IS) of 8.64 and a Fréchet Inception Distance (FID) of 11.63. Notably, these results surpass those of several models explicitly designed for image generation, including DDPM (fixed isotropic) and SNGAN~\cite{miyato2018spectral}, in both IS and FID metrics. This demonstrates the effectiveness and generalization capability of our model beyond its primary design scope.



\section{Conclusion}
In this study, we propose a novel paradigm, the \textbf{Discrete Markov Bridge} (\textbf{DMB}), which combines the strengths of variational methods with the capabilities of discrete diffusion models. We provide theoretical guarantees to substantiate the feasibility and effectiveness of the proposed \emph{Matrix}-learning process and prove the convergence of the DMB algorithm. In addition to our theoretical contributions, we conduct extensive empirical evaluations on the Text8 and CIFAR-10 datasets. The experimental results indicate that \textbf{DMB} not only surpasses existing baselines such as SEDD \cite{lou2024discrete} in text modeling tasks, but also achieves competitive performance in image modeling on CIFAR-10, thereby demonstrating its potential as a unified framework for discrete representation learning.


\newpage
\section*{Acknowledgement}
We thank Junqi Wang from BIGAI for inspiration of discovering the matrix and Jianwen Xie from Lambda for discussion. 
{
\small

\bibliographystyle{unsrtnat}
\bibliography{mybib}
}

\newpage
{\Huge \bf Appendices \bigskip}

\begin{appendices}

\DoToC

\clearpage

\section{Proof of Conservation of the Sum}
\label{supp: sec: proof_conservation}
\begin{proposition}[Conservation of the Sum]
For two arbitrary vectors $\phi, \mu \in \R^{d}$, rate transition matrix $Q \in \mathcal{R}^{d \times d}$, if $\phi = \mu \exp^{Q}$, then 
$$
\sum_{i = 1}^d \phi[i] = \sum_{i = 1}^d \mu[i]
$$
\end{proposition}
\begin{proof}
    As $\phi = \mu \exp^{Q}$, 
    $$
    \phi(i) = \sum \limits_{j} \mu(j) (\exp\{Q\})_{j,i}
    $$
    Therefore,
    $$
    \sum \limits_{i} \phi(i) = \sum \limits_{i} \sum \limits_{j} \mu(j) (\exp\{Q\})_{j,i}
    $$
   As we have
    $$
    \sum \limits_{j} (\exp\{Q\})_{i,j} = 1
    $$
    Thus,
    $$
    \sum \limits_{i} \phi(i) = \sum \limits_{j}  \mu(j) \sum \limits_{i}(\exp\{Q\})_{j,i} = \sum \limits_{j} \mu(j)
    $$
\end{proof}

\section{Proof of Accessibility}
\label{supp: sec: proof_accessibility}
\subsection{Proof of Lemmas}
\begin{lemma}
\label{supp: lemma: Q}
Let matrix $Q \in \mathbb{R}^{d \times d}$ and hold the following form:
$$
Q = 
\begin{bmatrix}
    -\sum \limits_{i = 1}^{n - 1}a_i & a_1 & a_2 & \ldots & a_{n-2} & a_{n-1}\\
    0 & -\sum \limits_{i = 2}^{n - 1}a_i & a_2 & \ldots & a_{n-2} & a_{n-1}\\
    0 & 0 &-\sum \limits_{i = 3}^{n - 1}a_i & \ldots & a_{n-2} & a_{n-1} \\
    \ldots & \ldots & \ldots & \ldots & \ldots & \ldots \\
    0 & 0 & 0 & \ldots & - a_{n-1} & a_{n-1} \\
    0 & 0 & 0 & 0 & 0 & 0 
\end{bmatrix}
$$

then $Q$ can be diagonalized in the following form:
$$
Q = U \Lambda U^{-1}
$$,where $U = \begin{bmatrix}
   1  & 1 & 1 & \ldots & 1\\
    0  & 1 & 1 & \ldots & 1\\
    0  & 1 & 1 & \ldots & 1 \\
    \ldots & \ldots & \ldots & \ldots & \ldots \\
    0 & 0 & 0 & \ldots & 1
\end{bmatrix}$, $\Lambda = diag(\{-\sum \limits_{i = 1}^{n - 1}a_i, -\sum \limits_{i = 2}^{n - 1}a_i, \ldots, - a_{n-1}, 0\})$
\end{lemma}

\begin{proof}
Q = $\begin{bmatrix}
   1  & 1 & 1 & \ldots & 1\\
    0  & 1 & 1 & \ldots & 1\\
    \ldots  & \ldots & \ldots & \ldots & \ldots \\
    0  & 0 & \ldots & 1 & 1 \\
    0  & 0 & 0 & 0 & 1
\end{bmatrix}diag(\{-\sum \limits_{i = 1}^{n - 1}a_i, -\sum \limits_{i = 2}^{n - 1}a_i, \ldots, - a_{n-1}, 0\}) \begin{bmatrix}
   1  & -1 & 0 & \ldots & 0\\
    0  & 1 & -1 & \ldots & 0\\
    \ldots  & \ldots & \ldots & \ldots & \ldots \\
    0  & 0 & \ldots & 1 & -1 \\
    0  & 0 & 0 & 0 & 1
\end{bmatrix}$
\end{proof}

\begin{lemma}
\label{supp: lemma:ineq_of_mass}
For arbitrary distribution $p, q \in \mathbb{R}^{1 \times d}$, there exists an permutation matrix $A$ such that:
\begin{equation}
   \frac{p_1'}{q_1'} \leq \frac{p_1' + p_2'}{q_1' + q_2'} \leq \ldots \leq \frac{\sum \limits_{i = 1}^k p_i'}{\sum \limits_{i = 1}^k q_i'} \leq \ldots \leq \frac{\sum \limits_{i = 1}^n p_i'}{\sum \limits_{i = 1}^n q_i'} = 1
\label{eq:ineq_of_mass} 
\end{equation}
where $p' = p A$, $q'=qA$, $p_i'$ is the $i$-th entry of $p'$
\end{lemma}

\begin{proof}
It's obvious that there exists a permutation matrix $A$ which can sort $\frac{p_i}{q_i}$ ascendly, i.e.:
$$
\frac{p_i'}{q_i'} \leq \frac{p_{i+1}'}{q_{i+1}'}
$$, where $p':= pA, q':=qA$, and the corner mark $i$ refer to the $i$-th entry.

Also, we can demonstrate that:
\begin{equation*}
    \frac{a_1}{b_1} \leq \frac{a_2}{b_2} \Rightarrow \frac{a_1}{b_1} \leq \frac{a_1 + a_2}{b_1 + b_2} \leq \frac{a_2}{b_2}  \tag{$\triangle$}
\end{equation*}

The inequality we need to prove is:
    $$
    \frac{\sum \limits_{i=1}^k p_i'}{\sum \limits_{i=1}^k q_i'} \leq \frac{\sum \limits_{i=1}^{k + 1} p_i'}{\sum \limits_{i=1}^{k+1} q_i'}
    $$
and it's sufficient to proving the following inequality:
$$
\frac{\sum \limits_{i=1}^k p_i'}{\sum \limits_{i=1}^k q_i'} \leq \frac{p_{k+1}'}{q_{k+1}'}
$$

We then start to prove the inequality by induction.

$k=1$: Let $a_1=p_1', a_2=p_2', b_1=q_1', b_2=q_2'$, and by using inequality $\triangle$, the statement is proved.

$k+1$: By induction: 
$$
\frac{\sum \limits_{i=1}^k p_i'}{\sum \limits_{i=1}^k q_i'} \leq \frac{p_{k+1}'}{q_{k+1}'}
$$

By leveraging inequality $\triangle$:
$$
\frac{\sum \limits_{i=1}^{k+1} p_i'}{\sum \limits_{i=1}^{k+1} q_i'} \leq \frac{p_{k+1}'}{q_{k+1}'}
$$

As $ \frac{p_{k+1}'}{q_{k+1}'} \leq  \frac{p_{k+2}'}{q_{k+2}'}$:
$$
\frac{\sum \limits_{i=1}^{k+1} p_i'}{\sum \limits_{i=1}^{k+1} q_i'}  \leq  \frac{p_{k+2}'}{q_{k+2}'}
$$

Thus the lemma is proved.
\end{proof}

\begin{lemma}
\label{supp: lemma: rate}
    Let $Q \in \mathbb{R}^{d \times d}$ be a rate transition matrix, $A \in \mathbb{R}^{d \times d}$ be a permutation matrix, then $AQA^{-1}$ is a rate transition matrix.
\end{lemma}

\begin{proof}
As every permutation matrix can be expressed as the products of elementary matrices, we denote:
$$
A = \prod \limits_{k=N_A}^{1} T_{ij}^{(k)} = T_{ij}^{(N_A)}T_{ij}^{(N_A - 1)} \ldots T_{ij}^{(1)}
$$, where $T_{ij}$ is the elementary matrix obtained by swapping row $i$ and row $j$ of the identity matrix, $N_A \in \mathbb{R}$

Therefore:
$$
AQA^{-1} = (\prod \limits_{k=N_A}^{1} T_{ij}^{(k)} ) Q (\prod \limits_{k=1}^{N_A} T_{ij}^{(k)})
$$

For a single pair of transformation, i.e. $T_{ij}^{(k)} Q T_{ij}^{(k)}$, the row sums remain unchanged, and the diagonal elements is still the diagonal elements after transformation, thus $AQA^{-1}$ is a rate transition matrix.
\end{proof}
\subsection{Proof of the theorem}
\begin{theorem}[Accessibility]
\label{supp: thm: accessibility}
For two arbitrary discrete distributions $p,q \in \mathbb{R}^{d}$, there exists a rate transition matrix $Q \in \mathbb{R}^{d \times d}$ such that:
$$
p = q e^Q
$$
\end{theorem}

\begin{proof}
    By Lemma \ref{lemma:ineq_of_mass}, there exists permutation matrix $A$ which satisfies inequality \ref{eq:ineq_of_mass}, and we denote:
    $$
    p' := pA
    $$
    $$
    q' := qA
    $$

    Suppose:
    $$
    Q := A Q' A^{-1}
    $$, where $Q' = \begin{bmatrix}
    -\sum \limits_{i = 1}^{n - 1}a_i & a_1 & a_2 & \ldots & a_{n-2} & a_{n-1}\\
    0 & -\sum \limits_{i = 2}^{n - 1}a_i & a_2 & \ldots & a_{n-2} & a_{n-1}\\
    \ldots & \ldots & \ldots & \ldots & \ldots & \ldots \\
    0 & 0 & 0 & \ldots & - a_{n-1} & a_{n-1} \\
    0 & 0 & 0 & 0 & 0 & 0 
\end{bmatrix} = U \Lambda U^{-1}$, $U$ is all one upper triangle matrix, and $\Lambda = diag(\{-\sum \limits_{i = 1}^{n - 1}a_i, -\sum \limits_{i = 2}^{n - 1}a_i, \ldots, - a_{n-1}, 0\})$ 

Denote:
$$
p'' := p' U = [p_1', p_1' + p_2', \ldots, \sum \limits_{i = 1}^{n - 1}p_i', 1]
$$
$$
q'' := q' U = [q_1', q_1' + q_2', \ldots, \sum \limits_{i = 1}^{n - 1}q_i', 1]
$$

Thus the solution of $p = qe^Q$ can be obtained by solving:
$$
p'' = q'' e^{\Lambda}
$$, where $e^{\Lambda} = diag(\{e^{-\sum \limits_{i = 1}^{n - 1}a_i}, e^{-\sum \limits_{i = 2}^{n - 1}a_i}, \ldots, e^{- a_{n-1}}, 1\})$
Solving the equation:
$$
a_k = \ln \frac{\sum \limits_{i = 1}^{k + 1} p_i'}{\sum \limits_{i = 1}^{k + 1} q_i'} - \ln \frac{\sum \limits_{i = 1}^{k} p_i'}{\sum \limits_{i = 1}^{k} q_i'}
$$
and specifically,
$$
a_{n-1} = - \ln \frac{\sum \limits_{i = 1}^{n-1} p_i'}{\sum \limits_{i = 1}^{n-1} q_i'}
$$
By the inequality \ref{eq:ineq_of_mass} which $p', q'$ satisfies and the monotonicity of the $\ln()$ function, $a_k \geq 0, \forall k$, and thus $Q'$ is a rate transition matrix

Transfering the solution of $p'' = q''e^\Lambda$ back, we obtain:
$$
Q = AU\Lambda U^{-1} A^{-1} = A Q' A^{-1}
$$ 
and by Lemma \ref{lemma: rate}, Q is a rate transition matrix.
\end{proof}

\section{Proof of Supervision of \textit{Score}-learning}
\label{supp: sec: proof_supervision}
\begin{proposition}[Supervision of \textit{Score}-learning]
    Suppose $\mQ^{(t)}$'s elements are non-zeros, the training objective is depicted as in \Cref{eq: score entropy loss}, then the optimality of the score model $s_{\theta^*}(x_t,t)_b$ satisfies:
    $$
    s_{\theta^{*}}(x_t,t)_y = \mathbb{E}_{x_0 \sim \mu_{0|t}(\cdot|x_t)}[\frac{p_{t|0}(y|x_0)}{p_{t|0}(x_t|x_0)}] = \frac{\sum \limits_{x_0}\mu(x_0) p_{t|0}(y|x_0)}{\sum \limits_{x_0}\mu(x_0) p_{t|0}(x_t|x_0)}
    $$
\end{proposition}
\begin{proof}
\begin{equation}
\begin{aligned}
J_{score} = & \int_{0}^{T} \E_{x_0 \sim \mu, x_t \sim p_{t|0}(x_t|x_0)}\bigl[\sum \limits_{y \neq x_t} \mQ^{(t)}_{y, x_t}\Biggl(s_\theta(x_t, t)_{y} - \frac{p_{t|0}(y|x_0)}{p_{t|0}(x_t|x_0)}\\
& + \frac{p_{t|0}(y|x_0)}{p_{t|0}(x_t|x_0)} \bigl( \log s_\theta(x_t, t)_y -\log(\frac{p_{t|0}(y|x_0)}{p_{t|0}(x_t|x_0)}) \bigl)\Biggl)\bigl] dt
\end{aligned}
\notag
\end{equation}
Therefore, with a little abuse of notation, we have

\begin{equation}
\begin{aligned}
\argmin \limits_{\theta} J_{score} = & \argmin \limits_{\theta} \int_{0}^{T} \E_{x_0 \sim \mu, x_t \sim p_{t|0}(x_t|x_0)}\bigl[\sum \limits_{b \neq x_t} \mQ^{(t)}_{y, x_t}\Biggl(s_\theta - 
 \frac{p_{t|0}(y|x_0)}{p_{t|0}(x_t|x_0)} \log s_\theta\Biggl)\bigl] dt \\
 = & \argmin \limits_{\theta} \underbrace{\int_{0}^{T} \E_{x_t \sim \mu_t}\bigl[\sum \limits_{y \neq x_t} \mQ^{(t)}_{y, x_t}\Biggl(s_\theta - \mathbb{E}_{x_0 \sim \mu_{0|t}}[\frac{p_{t|0}(y|x_0)}{p_{t|0}(x_t|x_0)}] \log s_\theta\Biggl)\bigl] dt}_\text{$\mathcal{L}$} \\
\end{aligned}
\notag
\end{equation}

$$
\frac{\partial \mathcal{L}}{\partial s_\theta} = \int_{0}^{T} \E_{x_t \sim \mu_t}\bigl[\sum \limits_{y \neq x_t} \mQ^{(t)}_{y, x_t}\Biggl( 1- \mathbb{E}_{x_0 \sim \mu_{0|t}}[\frac{p_{t|0}(y|x_0)}{p_{t|0}(x_t|x_0)}] \frac{1}{s_\theta}\Biggl)] dt
$$

As $\mQ^{(t)}$'s elements are non zeros, therefore 
$$\mQ^{(t)}_{y, x_t} > 0, \forall y \neq x_t$$ 
$$
\frac{\partial \mathcal{L}}{\partial s_\theta} = 0 \Longleftrightarrow 1- \mathbb{E}_{x_0 \sim \mu_{0|t}}[\frac{p_{t|0}(y|x_0)}{p_{t|0}(x_t|x_0)}] \frac{1}{s_\theta} = 0
$$
Therefore, the optimality of $s_\theta$ satisfies:
$$
s_{\theta^{*}}(x_t,t)_y = \mathbb{E}_{x_0 \sim \mu_{0|t}(\cdot|x_t)}[\frac{p_{t|0}(y|x_0)}{p_{t|0}(x_t|x_0)}]
$$
Furthermore, as $\mu_{0|t}(x_0|x_t) = \frac{\mu(x_0) p_{t|0}(x_t|x_0)}{\sum \limits_{x_0} \mu(x_0) p_{t|0}(x_t|x_0)}$, we have
$$
s_{\theta^{*}}(x_t,t)_y = \frac{\sum \limits_{x_0}\mu(x_0) p_{t|0}(y|x_0)}{\sum \limits_{x_0}\mu(x_0) p_{t|0}(x_t|x_0)}
$$
\end{proof}

\section{Proof of Convergence}
\label{sec: proof_convergence}
\subsection{Proof of Lemmas}
\begin{lemma}
\label{supp: lemma: KL monotonicity}
For a random variable $X_0 \in \mathbb{R}^n$ with arbitrary two distributions $p_0, p_0'$, the transition kernel is $p_{t|0}(x_t|x_0)$. We denote 
$$
p_t(x_t) := \sum \limits_{x_0} p_0(x_0) p_{t|0}(x_t|x_0)
$$
$$
p_t'(x_t) := \sum \limits_{x_0} p_0'(x_0) p_{t|0}(x_t | x_0)
$$
Then we have:
$$
D_{KL}(p_t || p_t') \leq D_{KL}(p_0 ||p_0')
$$
\end{lemma}

\begin{proof}
\begin{equation}
\begin{split}
D_{KL}(p_{0,t}(\cdot, \cdot) || p_{0,t}'(\cdot, \cdot)) & = \sum \limits_{x_0, x_t} p_{0,t}(x_0, x_t) \log \frac{p_{0,t}(x_0, x_t)}{p_{0,t}'(x_0, x_t)} \\
 & = \sum \limits_{x_0, x_t} p_{0,t}(x_0, x_t) \log \frac{p_{t|0}(x_t|x_0)p_0(x_0)}{p_{t|0}(x_t|x_0)p_0'(x_0)} \\
 & = D_{KL}(p_{0}|| p_{0}')
\end{split}
\notag
\end{equation}
Using the chain rule for KL divergence:
$$
D_{KL}(p_t ||p_t') = D_{KL}(p_{0,t}(x_0, x_t) || p_{0,t}'(x_0, x_t)) - \mathbb{E}_{p_t}[D_{KL}(p_{0|t}(x_0|x_t) || p_{0|t}'(x_0|x_t)]
$$
As KL divergence is greater than zero, we have:
$$
D_{KL}(p_t || p_t') \leq D_{KL}(p_{0,t}(x_0, x_t) || p_{0,t}'(x_0, x_t))  = D_{KL}(p_{0}|| p_{0}')
$$
\end{proof}

\subsection{Proof of the theorem}
\begin{theorem}[Convergence of the algorithm]
\label{supp: thm: convergence}
    If we assume optimality is achieved in every epoch of the forward process and the reverse process, and we denote the $k$-th epoch estimation of $\mu$ as $p_0$, then 
    $\lim_{k \to \infty} D_{KL}(\mu||p_0^{(k)})$ converges.

\end{theorem}

\begin{proof}


According to the assumption that each subprocess reaches its optimum,
$$
\mu = \mu p_{T|0}^{(k)} p_{0|T}^{(k); \leftarrow}
$$
$$
p_0^{(k+1)} = p_0^{(k)}p_{T|0}^{(k)} p_{0|T}^{(k); \leftarrow}
$$
Therefore, by using Lemma \ref{supp: lemma: KL monotonicity} twice:
$$
D_{KL}(\mu || p_0^{(k)}) \geq D_{KL}( \mu p_{T|0}^{(k)} || p_0^{(k)}p_{T|0}^{(k)}) \geq D_{KL}(\mu p_{T|0}^{(k)} p_{0|T}^{(k); \leftarrow} || p_0^{(k)}p_{T|0}^{(k)} p_{0|T}^{(k); \leftarrow})
$$
Therefore,
$$
D_{KL}(\mu||p_0^{(k)}) \geq D_{KL}(\mu||p_0^{(k+1)})
$$
As KL divergence is greater than zero, then \[\lim_{k \to \infty} D_{KL}(\mu||p_0^{(k)})\] converges.
\end{proof}





\section{Derivation of Matrix Exponential Calculation}
\label{supp: proof of matrix exponential calculation}
\begin{proposition}
    For a matrix $Q  \in \mathbb{R}^{n \times n}$ and a non-degenerate matrix $D \in \mathbb{R}^{n \times n}$, we have:
    $$
    \exp\{DQD^{-1}\} = D \exp\{Q\}D^{-1}
    $$
\end{proposition}
\begin{proof}
    According to the definition of matrix exponential,
    $$
    \exp\{DQD^{-1}\} = I + \sum_{i=1}^{\infty} (DQD^{-1})^i
    $$
    As $(DQD^{-1})^i = DQ^iD^{-1}$, 
    $$
    \exp\{DQD^{-1}\} = I + \sum_{i=1}^{\infty}DQ^iD^{-1} = D(I + \sum_{i=1}^\infty Q^i)D^{-1} = D\exp\{Q\}D^{-1}
    $$
\end{proof}
\section{Derivation of KL term calculation proposition}
\label{supp: proof of KL term cal}
The full bound \citep{meng2023concretescorematchinggeneralized, campbell2022continuous} is as follows:

$$
\mathbb{E}_{x_0 \sim \mu}[-\log p_{0; \theta}(x_0)] \leq J_{score} + \mathbb{E}_{x_0 \sim \mu} [D_{KL}\bigl(p_{T|0; \alpha}(x_T|x_0) || \phi \bigl)]
$$
, where
\begin{equation}
\begin{aligned}
J_{score} \triangleq & \int_{0}^{T} \E_{x_0 \sim \mu, x_t \sim p_{t|0}(x_t|x_0)}\bigl[\sum \limits_{b \neq x_t} Q^{(t)}_{b, x_t}\Biggl(s_\theta(x_t, t)_{b} - \frac{p_{t|0}(b|x_0)}{p_{t|0}(x_t|x_0)}\\
& + \frac{p_{t|0}(b|x_0)}{p_{t|0}(x_t|x_0)} \bigl( \log s_\theta(x_t, t)_b -\log(\frac{p_{t|0}(b|x_0)}{p_{t|0}(x_t|x_0)}) \bigl)\Biggl)\bigl] dt
\end{aligned}
\notag
\end{equation}
However, unlike previous works, the second term, which is the $KL$ term should be considered, and it seems impossible to compute. Fortunately, certain characteristics of the \textit{Matrix}-learning process can be used to justify a computable form for the second term. Suppose the text sequence holds $d$ dimensions, $i.e. x \in \mathbb{R}^{d}$, then the characteristics can be described as follows:
\begin{itemize}
    \item Independent Evolution:
$$
p_{T|0; \alpha}(x_T|x_0) = \prod_{i=1}^{d}p_{T|0; \alpha}(x_T^{(i)} | x_0^{(i)})
$$
\item Independent Terminal:
$$
\phi(x_T) = \prod_{i = 1}^d p_T(x_T^{(i)})
$$

\end{itemize}

As a result, we provide a computable form for the KL term.

\begin{proposition}
   $ D_{KL}\bigl(p_{T|0;\alpha}(x_T|x_0) || p_T(x_T) \bigl) = \sum \limits_{i = 1}^{d}D_{KL}\bigl(p_{T|0;\alpha}(x_T^{(i)} | x_0^{(i)}) || p_T(x_T^{(i)}) \bigl)$
\end{proposition}

\begin{proof}
By independent evaluation and independent terminal, we have
\begin{equation}
    \begin{aligned}
        D_{KL}\bigl(p_{T|0;\alpha}(x_T|x_0) || p_T(x_T) \bigl) &= \sum \limits_{x_T} p_{T|0; \alpha}(x_T | x_0) \log \frac{p_{T|0; \alpha} (x_T | x_0)}{\phi} \\
        &= \sum \limits_{x_T^{(1)}, x_T^{(2)}, \ldots, x_T^{(d)}}  p_{T|0;\alpha}(x_T | x_0) \sum \limits_{i=1}^d\log \frac{p_{T|0;\alpha}(x_T^{(i)} | x_0^{(i)})}{p_T(x_T^{(i)})} \\
        &= \sum \limits_{i=1}^d \sum \limits_{x_T^{(1)}, x_T^{(2)}, \ldots, x_T^{(d)}}  p_{T|0;\alpha}(x_T | x_0) \log \frac{p_{T|0;\alpha}(x_T^{(i)} | x_0^{(i)})}{p_T(x_T^{(i)})}\\
        &= \sum \limits_{i=1}^d \sum \limits_{x_T^{(i)}}p_{T|0;\alpha} (x_T^{(i)} | x_0^{(i)}) \log \frac{p_{T|0;\alpha}(x_T^{(i)} | x_0^{(i)})}{p_T(x_T^{(i)})} \\
        &= \sum \limits_{i=1}^d D_{KL}\bigl(p_{T|0; \alpha} (x_T^{(i)}|x_0^{(i)}) || p_T(x_T^{(i)} \bigl)\\
    \end{aligned}
    \notag
\end{equation}
\end{proof}

\section{Additional Experimental details}
\label{supp_sec: exp details}
\subsection{Model Details}
In terms of text modeling, for \emph{Matrix}-learning, the $Q_\alpha$ matrix is initialized as follows:
$$
a_i = 0, \forall i = 1,2,3,\ldots, n-2
$$
$$
a_{n-1} = 1
$$
The model is kept the same as SEDD \cite{lou2024discrete}. 

As for image modeling, for \emph{Matrix}-learing, the $Q_\alpha$ matrix is initialized as follows:
$$
a_i = 1e-5, \forall i = 1,2,3,\ldots, n-2
$$
The model is kept the same as SEDD \cite{lou2024discrete}. 
\subsection{Training Details}
The model is trained with a batch size of $512$ and trained with a learning rate of $3 \times 10^{-4}$ (Adam optimizer) on $8$ 4090 24GB GPUs. Both the \emph{Matrix}-learning as well as the \emph{Score}-learning are trained with the AdamW \cite{loshchilov2019decoupledweightdecayregularization}. 
Training start with a weight decay factor 0.01, which then turn to 0 in the 7,900,000 step for text8.



\section{Limitations and Societal Impact}
\label{supp_sec: limitations_and_societal_impact}
In this work, the \ac{CTDMB} framework primarily relies on the evidence lower bound (ELBO) for both training and evaluation. However, given that \Cref{thm: convergence} is not dependent on the specific form of the loss function, it is theoretically possible to derive other bounds for training. This flexibility opens new avenues for optimizing \ac{CTDMB} under different theoretical and practical settings. Furthermore, we haven't provided a theorem focusing on optimality, which may be done for future work. As for societal impact, our work focus on foundation learning algorithms, which doesn't hold direct societal impact.
\end{appendices}

\end{document}